\documentclass[orivec]{llncs}

\usepackage[english]{my-shortcuts}
\usepackage{srcltx,algorithm,algorithmic}
\usepackage{a4wide}

\usepackage{enumitem}

\usepackage{tikz}
\begin{document}

\title
{Convex recovery of tensors using nuclear norm penalization}
\author{St\'ephane Chr\'etien {\tiny and} Tianwen Wei}

\institute{Laboratoire de Math\'ematiques de Besancon\\ 
Universit\'e de Franche-Comt\'e\\
16 route de Gray,\\
25000 Besancon, France} 

\maketitle

\begin{abstract}
The subdifferential of convex functions of the singular spectrum of real matrices has been widely studied in matrix analysis, optimization and automatic control theory.  
Convex analysis and optimization over spaces of tensors is now gaining much interest due to its potential applications to signal processing, statistics and engineering. 
The goal of this paper is to
present an applications to the problem 
of low rank tensor recovery based on linear random measurement by extending the results of Tropp \cite{Tropp:ArXiv14} to the tensors setting. 
\end{abstract}


\section{Introduction}
\subsection{Background}
Tensors have been recently a subject of great interest in the applied mathematics community.  We refer to  \cite{KOLDA,Landsberg:Tensors12} for a modern reference on this subject. 
%
Many applications of tensors are based on solving tensor related optimization problems, such as minimizing certain norms under linear constraints. Such problems
have been recently successfully addressed in the 2D setting, i.e. for matrices, by the statistics, signal processing, inverse problems and automatic control communities 
in particular. Two of the reasons for this rapid growth of interest in the application of matrix norms to penalized estimation problems is that some norms 
promote spectral sparsity and that much work had been done in the fields of matrix analysis and convex analysis to analyze the subdifferential of such norms; 
see for example \cite{Watson:LAA92} and \cite{Lewis:JCA95}. 
Our goal in the present paper is to extend previous results on matrix norms to the tensor setting. 
In particular, we propose a general study of the subdifferential of certain convex functions of the spectrum of real  tensors and apply our results to the 
computation of the subdifferential of useful and natural matrix norms.  We also present an application of our formulas to the problem of low rank tensor recovery 
using sparsity promoting norm minimization under random linear constraints, a natural extension of previous works by Tropp \cite{Tropp:ArXiv14}.

\subsection{Notations}
For any convex function $f$ : $\mathbb R^{n}\mapsto \mathbb R\cup \{+\infty\}$, the conjugate function $f^*$ associated to $f$ is defined by 
\bean
f^*(g) \defeq \sup_{x \in \mathbb R^n}\quad \langle g, x  \rangle - f(x).
\eean
The subdifferential of $f$ at $x \in \mathbb R^n$ is defined by 
\bean 
\partial f & \defeq & \left\{ g \in \mathbb R^n \mid \forall y, \in \mathbb R^n \hspace{.3cm} 
f(y)\ge f(x)+\langle g,y-x\rangle  \right\}. 
\eean 
Moreover, it is well known (see e.g. \cite{HULL:Springer96}) that 
$g\in\partial f(x)$ if and only if
\bean
f(x) + f^*(g) & = &  \langle g, x\rangle. 
\eean

In the present paper, a tensor represented by a multi-dimensional array in $\mathbb R^{d_1\times \cdots\times d_D}$. 
Let $D$ and $n_1,\ldots,n_D$ be positive integers. Let $\Xc \in \mathbb R^{n_1\times \cdots \times n_D}$ denote a 
$D$-dimensional tensor. If $n_1=\cdots=n_D$, then we say that $\Xc$ is cubic. The set of $D$-mode cubic tensors 
will be denoted by $\Rb^{n\times \cdots \times n}$, where $D$ will stay implicit. 
For any index set $C \subset \{1,\ldots,n_1\} \times \cdots \times \{1,\ldots,n_D\}$, 
$\Xc_{C}$ will denote the subarray $(\Xc_{i_1,\ldots,i_D})_{(i_1,\ldots,i_D) \in C}$.

\section{Basics on tensors}
\subsection{Tensor norms}

\subsubsection{The spectrum of a tensor}

Let us define the spectrum as the mapping which to any tensor $\Xc\in \Rb^{n\times \cdots \times n}$ 
associates the vector $\sigma(\Xc)$ given by  
\bean 
\sigma(\Xc) & \defeq & \frac1{\sqrt{D}} \: (\sigma^{(1)}(\Xc),\ldots,\sigma^{(D)}(\Xc)),
\eean 
where $\sigma^{(d)}(\Xc)$ denotes the vector consisting of the singular values of the mode-$d$ matricization
of $\Xc$.

\subsubsection{Norms of tensors}
Let $\Xc=(\Xc_{ijk})$ and $\Yc=(\Yc_{ijk})$ be tensors in $\Rb^{n_1\times \cdots \times n_D}$. We can define 
several tensor norms on $\Rb^{n_1\times \cdots \times n_D}$. The first one is a natural extension of the 
Frobenius norm or Hilbert-Schmidt norm from matrices to tensors. We start by defining the following scalar product 
on $\Rb^{n_1\times \cdots \times n_D}$: 
\bean
\langle \Xc, \Yc\rangle & \defeq & \sum_{i_1=1}^{n_1}\cdots  \sum_{i_D=1}^{n_D} \Xc_{i_1,\ldots,i_D} \Yc_{i_1,\ldots,i_D}.   
\eean 
Using this scalar product, we can define the following norm, which we call the Frobenius norm
\bean 
\|\Xc\|_F & \defeq & \sqrt{\langle \Xc,\Xc\rangle}.
\eean
One may also define an ``operator norm'' in the same manner as for matrices as follows
\bean 
\|\Xc\|_{\phantom{S}} &\defeq & 
\max_{{ u^{(d)} \in\Rb^{n_d}, \atop \|u^{(d)}\|_2=1,d=1,\ldots,D}}
\langle\Xc, u^{(1)}\otimes\cdots \otimes u^{(D)}\rangle
\eean 
We also define 
%
\bean
\|\Xc\|_{*} &  \defeq & \frac{1}{D} \sum_{d=1}^D \|\sigma^{(d)}\|_1. 
\eean

\subsection{Orthogonally decomposable tensors}
\label{odec}
The Orthogonally decomposable (ODEC) tensors are defined as follows
\begin{defi}
Let $\Xc$ be a tensor in $\Rb^{n_1\times \cdots \times n_D}$. If
\bea\label{defi1}\label{odec1}
\Xc & = & \sum_{i=1}^{r}\alpha_i\cdot u_i^{(1)}\otimes\cdots \otimes u_i^{(D)},
\eea
where $r\leq n_1 \wedge \cdots \wedge n_D$, $\alpha_1\geq \cdots \geq \alpha_r > 0$  and
$\{u_1^{(d)}, \ldots, u_r^{(d)}\}$  is a family of orthonormal vectors for $d=1,\ldots,D$, then we say
(\ref{defi1}) is an orthogonal decomposition of $\Xc$.
\end{defi}
Denote $\alpha=(\alpha_1,\ldots, \alpha_r, 0,\ldots, 0)$ in $\mathbb R^{n_1 \wedge \cdots \wedge n_D}$.
For each $d\in\{1,\ldots, D\}$, we may complete $\{u_1^{(d)}, \ldots, u_r^{(d)}\}$ with
$\{u_{r+1}^{(d)},\ldots, w_{n_d}^{(d)}\}$ so that matrix 
$U^{(d)}=(u_1^{(d)},\ldots, u_{n_d}^{(d)})\in\Rb^{n_d\times n_d}$ is orthogonal. 
Using $U^{(1)},\ldots, U^{(D)}$,
we may write (\ref{defi1}) as
\bea
\Xc \label{odec2}
& = & \Dc(\alpha)\times_1 U^{(1)} \times_2 U^{(2)} \cdots\times_D U^{(D)}.
\eea
where $\Dc=\diag(\alpha)$ is a diagonal tensor with the $i$th diagonal being $\alpha_i$ for $i=1,\ldots, r$ and the other diagonal entries being zero.
Note that representation (\ref{odec2}) is generally not unique unless $n_1=\cdots=n_D$ and $\alpha_1,\ldots,\alpha_r$ are all distinct.

It is easy to calculate the norms of ODEC tensors. 
\begin{prop}\label{spectrum1}
Let $\Xc$ be an orthogonally decomposable tensor and let
\bean
\Xc&=&\sum_{i=1}^{r}\alpha_i\cdot u_i^{(1)}\otimes\cdots \otimes u_i^{(D)},
\eean
be an orthogonal decomposition of $\Xc$. Then
\bean
\|\Xc\|  & = \alpha_1  \quad \mathrm{and} \quad
\|\Xc\|_* = & \sum_{i=1}^r \alpha_i.
\eean 
\end{prop}


\section{Further results on the spectrum}

In this section, we will present some further results on the spectrum such as the question of characterizing the 
image of the spectrum and the subdifferential of a function of the spectrum. 

\subsection{A technical prerequisite: Von Neumann's inequality for tensors}
Von Neumann's inequality says that for any two matrices$ X$ and $Y$ in $\mathbb R^{n_1\times n_2}$, we have 
\bean 
\langle X,Y\rangle & \le & \langle \sigma(X),\sigma(Y)\rangle, 
\eean 
with equality when the singular vectors of $X$ and $Y$ are equal, up to permutations when the singular values have multiplicity greater than one. This result has proved useful for the study of the subdifferential of unitarily invariant convex functions of the spectrum in the matrix case in \cite{Lewis:JCA95}.
In order to study the subdifferential of the norms of certain type of tensors, we will need a generalization this result to higher orders. 
This  was worked out in \cite{ChretienWei:ArXiv15}. Let us recall the containt of the main result of \cite{ChretienWei:ArXiv15}. 

\begin{defi}
We say that a tensor $\Sc$ is blockwise decomposable if there exists an integer $B$ and if, for all $d=1,\ldots,D$, there exists a partition $I_1^{(d)} \cup \ldots \cup I_B^{(d)}$ into 
disjoint index subsets of $\{1,\ldots,n_d\}$, such that $\Xc_{i_1,\ldots,i_D}=0$ if for all $b=1,\ldots,B$, $(i_1,\ldots,i_D) \not \in I_b^{(1)}\times \ldots \times I_b^{(D)}$.
\end{defi}
An illustration of this block decomposition can be found in Figure \ref{blocks}. 
The following result is a generalization of von Neumann's inequality from matrices to tensors. It is proved in \cite{ChretienWei:ArXiv15}. 
\begin{theo}\label{VN3}
Let $\Xc,\Yc\in \Rb^{n_1\times \cdots \times n_D}$ be tensors.  
Then for all $d=1,\ldots,D$, we have 
\bea\label{vonneumann1}
\langle \Xc, \Yc\rangle \leq \langle \sigma^{(d)}(\Xc), \sigma^{(d)}(\Yc) \rangle.
\eea
Equality in (\ref{vonneumann1}) holds simultaneously for all $d=1,\ldots,D$ if and only if there 
exist orthogonal matrices $W^{(d)}\in\Rb^{n_d\times n_d}$ for $d=1,\ldots, D$ 
and tensors $\Dc(\Xc),\Dc(\Yc)\in\Rb^{n_1\times\cdots\times n_D}$ such that
\bean
\Xc&=& \Dc(\Xc) \times_1 W^{(1)} \cdots \times_{D} W^{(D)}, \\
\Yc&=& \Dc(\Yc) \times_1 W^{(1)} \cdots \times_{D} W^{(D)},
\eean
where $\Dc(\Xc)$ and $\Dc(\Yc)$ satisfy the following properties:
 \begin{enumerate}
 \item[(i)] $\Dc(\Xc)$ and $\Dc(\Yc)$ are block-wise decomposable with the same number of blocks, which we will denote by $B$,
 \item[(ii)] the blocks $\{\Dc_b(\Xc)\}_{b=1,\ldots,B}$ (resp. $\{\Dc_b(\Yc)\}_{b=1,\ldots,B}$) on the diagonal of $\Dc(\Xc)$ (resp. $\Dc(\Yc)$) have the same sizes,
 \item[(iii)] for each $b=1,\ldots, B$ the two blocks $\Dc_b(\Xc)$ and $\Dc_b(\Yc)$ are proportional.
 \end{enumerate}
 \end{theo}

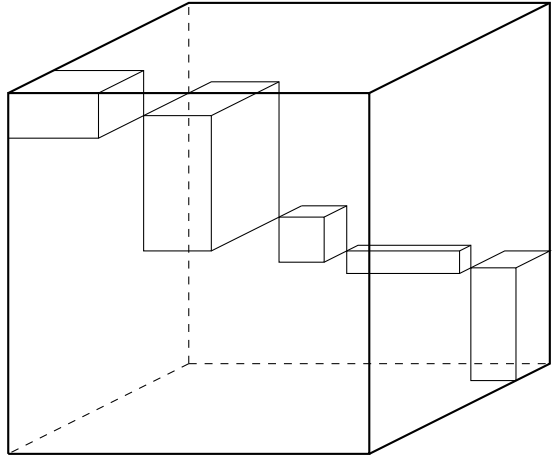
\begin{figure}
\begin{center}\begin{tikzpicture}[scale=0.60]

\draw[thick] (0,0) -- (0,8) -- (8,8)--(8,0)--(0,0);
\draw[thick] (8,0) -- (12,2) -- (12,10) -- (4,10) -- (0,8);
\draw[thick] (8,8) -- (12,10);
\draw[dashed] (4,2) -- (12,2);
\draw[dashed] (4,2) -- (4,10);
\draw[dashed] (4,2) -- (0,0);

\draw (0,7) -- (2,7) -- (2,8) -- (3, 8.5);
\draw (1, 8.5) -- (3, 8.5) -- (3, 7.5) -- (2, 7);


\draw (3, 7.5) -- (3, 4.5) -- (4.5,4.5) -- (4.5, 7.5) -- (3, 7.5);
\draw (3, 7.5) -- (4.5, 8.25) -- (6, 8.25) -- (6, 5.25) -- (4.5, 4.5);
\draw (4.5, 7.5) -- (6, 8.25);


\draw (6, 5.25) -- (6, 4.25) -- (7, 4.25) -- (7, 5.25) -- (6, 5.25);
\draw (7.5, 4.5) -- (7.5, 5.5) -- (6.5, 5.5);

\draw (6, 5.25) -- (6.5, 5.5);
\draw (7, 5.25) -- (7.5, 5.5);
\draw (7, 4.25) -- (7.5, 4.5);


\draw (7.5, 4.5) -- (7.5, 4) -- (10, 4) -- (10, 4.5) -- (7.5, 4.5);
\draw (7.5, 4.5) -- (7.75, 4.625) -- (10.25, 4.625)-- (10.25, 4.125)-- (10, 4);
\draw (10.25, 4.625) -- (10, 4.5);

\draw (10.25, 4.125) -- (10.25, 1.625) -- (11.25, 1.625) -- (11.25, 4.125) -- (10.25, 4.125);
\draw (10.25, 4.125) -- (11, 4.5) -- (12, 4.5) -- (11.25, 4.125);
\end{tikzpicture}
\end{center}
\caption{A block-wise diagonal tensor.}
\label{blocks}
\end{figure}

\subsection{Subdifferential for ODEC tensors}
\begin{theo} \label{conjugate2}
Let $f: \Rb^{n}\times\cdots\times\Rb^n \mapsto \Rb$ satisfy property 
\begin{align}
\label{symf}
f(s_1,\ldots,s_D) & = f(s_{\tau(1)},\ldots,s_{\tau(D)})
\end{align}
for all $\tau \in \mathfrak{S}_S$. 
 Then for all ODEC tensors $\Xc$, we have
\bea\label{114a}
(f\circ\sigma)^* (\Xc)= f^*(\sigma(\Xc))
\eea
\end{theo}
Using this result combined with von Neumann's inequality for tensors, one easily obtains the following corollary. 
\begin{coro}
Let $f: \Rb^{n}\times\cdots\times\Rb^n \mapsto \Rb$ satisfy property 
\begin{align}
\label{symf}
f(s_1,\ldots,s_D) & = f(s_{\tau(1)},\ldots,s_{\tau(D)})
\end{align}
for all $\tau \in \mathfrak{S}_S$. 
Let $\Xc$ be an ODEC tensor. Then necessary and sufficient conditions for an 
ODEC tensor $\Yc$ to belong to $\partial (f\circ\sigma)(\Xc)$ are
\begin{enumerate}
\item $\Yc$ has the same mode-$d$ singular spaces as $\Xc$ for all $d=1,\ldots,D$,
\item $\sigma(\Yc)\in\partial f(\sigma(\Xc))$.
\end{enumerate}
\end{coro}
\begin{coro}
Let
$\Xc =  \Dc(\alpha)\times_1 U^{(1)} \times_2 \cdots\times_D U^{(D)}$
be an ODEC tensor. Then
 the subdifferential $\partial\|\cdot\|_*(\Xc)$ includes the following set
\bean
\Omega=\left\{ \Dc(\boldsymbol{1})\times_1 U^{(1)} \times_2 \cdots\times_D U^{(D)}
+\Vc\,\, \Big|\,\, \|\Vc\|\leq 1, \,\,\Vc\times_i U^{(i)\TR}=0, \,\,i=1,\ldots,D\right\}.
\eean
\end{coro}

\section{Application to tensor recovery with gaussian measurements}
Let $\Xc^\#\in \Rb^{n_1\times n_2\times n_3}$ be an unknown true signal,  
$\Phi(\cdot):\Rb^{n_1\times n_2\times n_3}\mapsto \Rb^m$ be a known linear measurement mapping and
\bea
y=\Phi(\Xc^\#)+\xi
\eea
 be a noised vector of measurements in $\Rb^m$.

We focus on the following optimization problem:
\bea\label{928a}
\min_{\Xc} \|\Xc\|_*\quad\textrm{subject to } \|\Phi(\Xc) -y\|\leq \eta.
\eea
 Let $\hat{\Xc}$ be any solution of optimization problem (\ref{928a}). We are interested in giving a bound for
\bean
\|\hat{\Xc} - \Xc^\#\|_F.
\eean
The main tool of this section is the following result by Tropp \cite{Tropp:ArXiv14}:
\begin{theo}\label{TroppTheo1}
Assume that $\|\xi\|\leq \eta$. Then with probability at least $1- e^{-t^2/2}$, we have
\bean
\|\hat{\Xc} - \Xc^\#\|_F \leq \frac{2\eta}{[\sqrt{m-1} - w(\Ds(\|\cdot\|_*,\Xc^\#)) -t  ]_+},
\eean
where $[a]_+=\max\{a,0\}$ for any $a\in\Rb$.
\end{theo}
The quantity $w(\Ds(\|\cdot\|_*,\Xc^\#))$
denotes the conic Gaussian width $w(\cdot)$ of the descent cone $\Ds(\|\cdot\|_*,\Xc^\#)$.
The definitions of these notions are given as follows:
\begin{defi}
Let $K\in \Rb^d$ be a cone,  the conic Gaussian width $w(K)$ is defined as
\bean
w(K) = \Eb[\sup_{u\in K\cap \Sc^{d-1}} \langle g, u \rangle ],
\eean
where $g\sim \Nc(0,I)$ is a standard Gaussian vector and $\Sc^{d-1}$ denotes the unit sphere in $\Rb^d$.
\end{defi}
\begin{defi}
Let $f:\Rb^d\mapsto\bar{\Rb}$ be a proper convex function. The descent cone $\Ds(f,x)$ of the function $f$ at a point $x\in\Rb^d$ is defined as
\bean
\Ds(f,x)\defeq  \{\lambda u\,|\, \lambda>0, u\in\Rb^d,\, f(x+u)\leq f(x) \}.
\eean
\end{defi}

According to Theorem \ref{TroppTheo1}, the error bound of $\|\hat{\Xc} - \Xc^\#\|_F$ depends on the
conic Gaussian width $w(\cdot)$ of the descent cone $\Ds(\|\cdot\|_*,\Xc^\#)$.
The following result reveals that the latter is then closely related to the subdifferential of $\|\cdot\|_*$ at
$\Xc^\#$.
\begin{prop}
Assume that $\partial\|\Xc^\#\|$ is nonempty and does not contain the origin. Then
\bean
w^2(\Ds(\|\cdot\|_*,\Xc^\#)) \leq \Eb \inf_{\tau\geq 0} \mathrm{dist}^2_F(\Gc, \tau \partial \|\Xc^\#\|_*),
\eean
where $\Gc\in\Rb^{n_1\times n_2 \times n_3}$ is a tensor with i.i.d. random Gaussian entries and
\bean
\mathrm{dist}_F(\Gc, \tau \partial \|\Xc^\#\|_*)\defeq \inf_{\Yc\in\tau \partial \|\Xc^\#\|_* } \|\Gc - \Yc\|_F,
\eean
i.e. the distance between $\Gc$ and the set $\tau \partial \|\Xc^\#\|_*$.
\end{prop}
To derive a bound for $\|\hat{\Xc} - \Xc^\#\|_F$, we need to give an upper bound for
\bean
 \Eb \inf_{\tau\geq 0}\mathrm{dist}_F^2(\Gc, \tau \partial \|\Xc^\#\|_*).
\eean
The following result establishes such a bound in the case that $\Xc^\#$ is odec.
\begin{prop}
If $\Xc^\#$ is odec, then we have the following bound:
\bean
\Eb \inf_{\tau\geq 0}\mathrm{dist}_F^2(\Gc, \tau \partial \|\Xc^\#\|_*)&\leq& r^3 + r+ 3r(n_1+n_2+n_3 - 3r)  + r(n_1n_2+n_2n_3 + n_1n_3) \\
&&- r^2(n_1+n_2+n_3).
\eean
\end{prop}
\begin{proof}
If $\Xc^\#$ is orthogonally decomposable, i.e.
\bean
\Xc^\# 
&=& \sum_{i=1}^r \sigma_i u_i^{(1)}\otimes u_i^{(2)}\otimes u_i^{(3)}  \\
&=& \Dc(\sigma)\times_1 U^{(1)}\times_2 U^{(2)} \times_3 U^{(3)},
\eean
where $\Dc(\sigma)$ is a diagonal tensor with diagonal elements $\sigma=(\sigma_1,\ldots,\sigma_r)$ and 
$U^{(j)}=(u^{(j)}_1,\ldots,u^{(j)}_r)$ for $j=1,2,3$, then the subdifferential $\partial\|\cdot\|_*(\Xc^\#)$ includes the following set
\bea\label{subdiff}
\Omega=\left\{ \sum_{i=1}^r u_i^{(1)}\otimes u_i^{(2)}\otimes u_i^{(3)} 
+\Vc\,\, \Big|\,\, \|\Vc\|\leq 1, \,\,\Vc\times_i U^{(i)}=0, \,\,i=1,2,3.\right\}.
\eea
Hence
\bean
 \Eb \inf_{\tau\geq 0}\mathrm{dist}_F^2(\Gc, \tau \partial \|\Xc^\#\|_*)
\leq \Eb \inf_{\tau\geq 0}\mathrm{dist}_F^2(\Gc, \tau \Omega)
=\Eb \inf_{\tau\geq 0}\inf_{\Yc\in\Omega}\|\Gc -  \tau \Yc\|_F^2.
\eean
Note that $\Vc$ in (\ref{subdiff}) can also be characterized by
\bea\label{tensorv}
\Vc=\Tc \times_1 {U}^{(1)}_\perp\times_2 {U}^{(2)}_\perp \times_3 {U}^{(3)}_\perp,
\eea
where $\Tc\in\Rb^{(n_1-r)\times(n_2-r)\times (n_3-r)}$ is a tensor such that $\|\Tc\|\leq 1$ and 
${U}^{(i)}_\perp \in\Rb^{n_i\times (n_i-r)}$ is a matrix such that $\tilde{U}^{(i)}=(U^{(i)} | U^{(i)}_\perp)$ is orthogonal for $i=1,2,3$.
In view of (\ref{subdiff}) and (\ref{tensorv}), we assert that any $\Yc\in\Omega$ can be written as
\bean
\Yc=\Cc \times_1 \tilde{U}^{(1)}\times_2 \tilde{U}^{(2)} \times_3 \tilde{U}^{(3)}.
\eean
where tensor $\Cc$ is block-wise diagonal with two diagonal blocks $\Cc_{1}=\diag(\boldsymbol{1})\in\Rb^{r\times r\times r}$ and
$\Cc_2=\Tc\in\Rb^{(n_1-r)\times(n_2-r)\times (n_3-r)}$.

Because $\Gc\in\Rb^{n_1\times n_2 \times n_3}$ is a tensor with i.i.d.  random standard Gaussian entries,
for any orthogonal matrices $W^{(1)}, W^{(2)},W^{(3)}$ with appropriate size,
  tensor $\Gc \times_1 W^{(1)}\times_2 W^{(2)} \times_3 W^{(3)}$  still has i.i.d. standard Gaussian entries. Therefore, we may choose a coordinate system such that
\bean
\Eb \inf_{\tau\geq 0}\inf_{\Yc\in\Omega}\|\Gc -  \tau \Yc\|_F^2
=\Eb \inf_{\tau\geq 0}\inf_{\Cc\in\tilde{\Omega}}\|\Gc -  \tau \Cc\|_F^2,
\eean
where $\tilde{\Omega}$ denotes the set of block-wise diagonal tensors with two diagonal blocks $\Cc_{111}=\Dc(\boldsymbol{1})\in\Rb^{r\times r\times r}$ and
$\Cc_{222}\in\Rb^{(n_1-r)\times(n_2-r)\times (n_3-r)}$ verifying $\|\Cc_2\|\leq 1$.
Partitioning $\Gc$ in the same manner, we obtain 
\bean
\|\Gc -  \tau \Cc\|_F^2 = \|\Gc_{111} - \tau \Dc(\boldsymbol{1})\|_F^2 + \|\Gc_{222} - \tau\Tc\|^2_F
+\sum_{i,j,k=1  \atop i,j,k \textrm{ are not equal}}^2 \|\Gc_{i,j,k}\|_F^2.
\eean
Since $\Gc$ is a tensor with independent Gaussian entries, it follows that
\bean
\Eb \sum_{i,j,k=1  \atop i,j,k \textrm{ are not equal}}^2 \|\Gc_{ijk}\|_F^2
=r(n_1n_2+n_2n_3 + n_1n_3) - r^2(n_1+n_2+n_3).
\eean
Thus
\bean
\Eb \inf_{\tau\geq 0}\inf_{\Cc\in\tilde{\Omega}}\|\Gc -  \tau \Cc\|_F^2
&=& \Eb \inf_{\tau\geq 0}\inf_{\|\Cc_2\|\leq 1}\Big(\| \Gc_{111} - \diag(\tau) \|_F^2+\|\Gc_{222} -  \tau \Cc_2\|_F^2\Big)\\
&& + r(n_1n_2+n_2n_3 + n_1n_3) - r^2(n_1+n_2+n_3). \\
\eean
Choosing $\tau=\|\Gc_2\|$, we get
\bean
\Eb \inf_{\tau\geq 0}\inf_{\|\Cc_2\|\leq 1}\Big(\| \Gc_1 - \diag(\tau) \|_F^2+\|\Gc_2 -  \tau \Cc_2\|_F^2\Big)
\leq \Eb \| \Gc_1 - \diag(\|\Gc_2\|) \|_F^2 
\eean
Since
\bean
\Eb \| \Gc_1 - \diag(\|\Gc_2\|) \|_F^2 &=& r^3 + r\Eb \|\Gc_2\|^2 \leq r^3 + r + r\Big( \sqrt{n_1-r} + \sqrt{n_2-r} + \sqrt{n_3 -r}  \Big)^2 \\
&\leq &r^3 + r+ 3r(n_1+n_2+n_3 - 3r),
\eean
It follows that
\bean
\Eb \inf_{\tau\geq 0}\mathrm{dist}_F^2(\Gc, \tau \partial \|\Xc^\#\|_*)&\leq& r^3 + r+ 3r(n_1+n_2+n_3 - 3r)  + r(n_1n_2+n_2n_3 + n_1n_3) \\
&&- r^2(n_1+n_2+n_3).
\eean
\end{proof}
If the tensor is cubic, i.e. $n_i=n$ for $i=1,2,3$, then  we have with at least probability $1-e^{-t^2/2}$ that
\bean
\|\hat{\Xc} - \Xc^\#\|_F \leq \frac{2\eta}{[\sqrt{m-1} - (r^3 + r+ 9r(n - r)  + 3r n(n-r)) -t  ]_+}.
\eean

\bibliographystyle{amsplain}
\bibliography{Tensor}


\end{document}